\documentclass[letterpaper, 10 pt]{article}
\usepackage{url}
\usepackage{ifthen}
\usepackage{booktabs}       
\usepackage{nicefrac}      
\usepackage{microtype}     
\usepackage{varwidth}
\usepackage{wrapfig}
\usepackage{subcaption}
\usepackage{comment}
\usepackage{helvet}
\usepackage{courier}
\usepackage{nameref}
\usepackage{centernot}
\usepackage{verbatim}
\usepackage{array}
\makeatletter
\newcommand{\thickhline}{
    \noalign {\ifnum 0=`}\fi \hrule height 1pt
    \futurelet \reserved@a \@xhline
}
\newcolumntype{"}{@{\hskip\tabcolsep\vrule width 1pt\hskip\tabcolsep}}
\makeatother
\usepackage[english]{babel}
\usepackage{amsthm}
\theoremstyle{definition}
\newtheorem{definition}{Definition}
\newtheorem{theorem}{Theorem}
\newtheorem{lemma}[theorem]{Lemma}
\newtheorem{remark}{Remark}

\newtheorem{corollary}{Corollary}
\newtheorem{proposition}{Proposition}

\newcommand\independent{\protect\mathpalette{\protect\independenT}{\perp}}
\def\independenT#1#2{\mathrel{\rlap{$#1#2$}\mkern2mu{#1#2}}}
\usepackage[T1]{fontenc}
\usepackage[utf8]{inputenc}
\usepackage{amsmath}
\usepackage{amsthm}
\usepackage{amsfonts}
\usepackage{amssymb}
\usepackage{graphicx}
\usepackage{pifont}
\usepackage{mathtools}
\usepackage{appendix}
\usepackage{hhline}
\usepackage{times}
\usepackage{algorithm}
\usepackage{algorithmic}
\usepackage{multirow}
\usepackage{multicol}
\usepackage{bm}
\usepackage{epsfig}
\usepackage{epstopdf}
\usepackage{color}
\usepackage{soul}
\usepackage{mathtools}
\usepackage{bbm}
\usepackage{tikz}
\usepackage[margin=1 in]{geometry}

\makeatletter
\newcommand\footnoteref[1]{\protected@xdef\@thefnmark{\ref{#1}}\@footnotemark}
\makeatother

\usepackage[colorlinks=true, linkcolor=blue, anchorcolor=blue, citecolor=blue, bookmarks=false]{hyperref}
\usepackage{cleveref}
\usepackage{enumitem}

\usepackage{empheq}
\usepackage{lipsum}

\usepackage{titlesec}
\titleformat*{\section}{\large\bfseries}
\titleformat*{\subsection}{\large\bfseries}
\titleformat*{\subsubsection}{\bfseries}
\titleformat*{\paragraph}{\bfseries}
\titleformat*{\subparagraph}{\bfseries}

\numberwithin{equation}{section}
\allowdisplaybreaks

\theoremstyle{plain}

\usepackage{cancel}
\usepackage{blindtext}

\title{Information Theoretic Measures for Fairness-aware Feature Selection} 

\author{
{\normalsize Sajad Khodadadian}\footnote{H. Milton Stewart School of Industrial \& Systems Engineering, Georgia Institute of Technology, \href{mailto:skhodadadian3@gatech.edu}{skhodadadian3@gatech.edu}}\and
{\normalsize Mohamed Nafea}\footnote{Department of Electrical Engineering and Computer Science, University Of Detroit Mercy, \href{mailto:nafeamo@udmercy.edu}{nafeamo@udmercy.edu}}\and
{\normalsize AmirEmad Ghassami}\footnote{Department of Computer Science, Johns Hopkins University, \href{mailto:aghassa1@jhu.edu}{aghassa1@jhu.edu}}\and
{\normalsize Negar Kiyavash}\footnote{College of Management of Technology, {\'E}cole Polytechnique F{\'e}d{\'e}rale de Lausanne (EPFL), \href{mailto:negar.kiyavash@epfl.ch}{negar.kiyavash@epfl.ch}}
}

\begin{document}

\maketitle

\begin{abstract}
Machine learning algorithms are increasingly used for consequential decision making regarding individuals based on their relevant features. Features that are relevant for accurate decisions may however lead to either explicit or implicit forms of discrimination against unprivileged groups, such as those of certain race or gender. This happens due to existing biases in the training data, which are often replicated or even exacerbated by the learning algorithm. Identifying and measuring these biases at the data level is a challenging problem due to the interdependence among the features, and the decision outcome. In this work, we develop a framework for fairness-aware feature selection which takes into account the correlation among the features and the decision outcome, and is based on information theoretic measures for the accuracy and discriminatory impacts of features. In particular, we first propose information theoretic measures which quantify the impact of different subsets of features on the accuracy and discrimination of the decision outcomes. We then deduce the marginal impact of each feature using Shapley value function; a solution concept in cooperative game theory used to estimate marginal contributions of players in a coalitional game. Finally, we design a fairness utility score for each feature (for feature selection) which quantifies how this feature influences accurate as well as nondiscriminatory decisions. Our framework depends on the joint statistics of the data rather than a particular classifier design. We examine our proposed framework on real and synthetic data to evaluate its performance.
\end{abstract}

\section{Introduction}\label{sec:intro}
Machine learning algorithms are increasingly utilized in many domains of human life such as advertising, healthcare, loan assessment, job applications, and predictive policing  \cite{kinyanjui2019estimating,mahoney2007method,angwin2016machine}. While learning algorithms may help improve prediction accuracy in these applications, there is a growing concern about potential discriminatory practices against underrepresented groups. These practices often result due to existing biases in training data which are replicated (or even exacerbated) by the learning algorithm. Due to legal and/or ethical implications of these decisions\footnote{As a legal example, Title VII of the Civil Rights Act of 1964 prohibits employers from discriminating against employees on the basis of features such as race, gender, religion, etc.}, mere improvements in reducing risk or decision error does not suffice, and the system should also take into account the potential discriminatory consequences of a decision. 

We consider the fairness problem that arises in supervised learning. In this setting, the goal is to assign an accurate label (decision) to each individual based on a set of features she/he possesses, while hindering certain features, referred to as \emph{protected attributes}, from influencing the decisions, \cite{calmon2017optimized,ghassami2018fairness,hardt2016equality,zemel2013learning, jiang2020identifying}. The law recognizes two doctrines of discrimination: (1) Disparate treatment (explicit discrimination) where the protected attribute (group membership) is deliberately used to treat underrepresented groups differently, and (2) Disparate impact (implicit discrimination) where the protected attribute is hidden from the decision maker but other features act as proxies for the group membership, unfavoring those of certain demographic \cite{hardt2016equality,zemel2013learning,barocas2016big,feldman2015certifying,zafar2015fairness}). For example, redlining, a systematic denial of some services such as banking, insurance, or healthcare to residents of certain neighborhoods, was adopted for decades in certain urban areas in the United States; here, the zip code acts as a proxy for the individual's race  \cite{hunt2005redlining}.

A candid approach to alleviate discriminatory decisions is to identify features with significant proxy behaviour and prevent these features from influencing the decision outcome, see for example \cite{wang2018direction}. However, a feature that acts as a proxy for the sensitive attributes(s) may contain information necessary for accurate classification. The question then is how to simultaneously quantify the marginal impact of each feature both on the accuracy and discrimination of the system? By answering this question, one can choose to select the subset of features that optimally satisfies certain accuracy/fairness requirements. 

In this paper, we propose two information-theoretic measures that separately quantify the accuracy and discriminatory impact of subsets of features. Subsequently, we deduce the marginal impacts of each feature using Shapely-value analysis \cite{shapley1953value}. Introducing the measures for subsets of features, rather than for individual features, allows us to account for the interdependencies among the features, sensitive attributes, and decision outcome. Finally, we define a fairness-utility score for each feature which combines both impacts for each feature. Using these scores, we can choose to select the subset of features that optimally satisfy certain accuracy/fairness requirements.

Information-theoretic measures are advantageous because they (i) capture non-linear dependencies among the system variables \cite{cover2012elements} and (ii) allow the system variables, particularly the protected attribute(s) and the true outcome, to have arbitrary cardinalities (as opposed to \cite{hardt2016equality,zemel2013learning}). Further, our measures are defined with respect to the features, the sensitive attributes, and the true outcome. Thus, our method only requires knowledge of their joint statistics, and is not limited to a particular choice of the classifier. 

\subsection{Related work}\label{related}
{\bf{Fairness.}} The problem of algorithmic fairness and various notions for nondiscriminative learning, both in the context of supervised learning \cite{dwork2012fairness,feldman2015certifying,hardt2016equality}, and otherwise \cite{joseph2016rawlsian,joseph2016fairness}, have been the focus of recent work. Several methods have been proposed to address disparate impact (implicit discrimination), see \cite{kamishima2011fairness,dwork2012fairness,zemel2013learning,feldman2015certifying,hardt2016equality,kilbertus2017avoiding, wang2019repairing}. These references adopted different notions of fairness, and proposed algorithms to mitigate unwanted biases. Notions of fairness can be coarsely divided into (i) {\it{group fairness}} which typically fixes some protected attributes (group memberships) and requires parity for a certain statistical measure across these groups \cite{hardt2016equality,kleinberg2016inherent,kusner2017counterfactual}, and (ii) {\it{individual fairness}}, which requires that similar individuals should be treated similarly, for certain similarity measures between individuals, as well as between classification outcomes \cite{dwork2012fairness,zemel2013learning,joseph2016fairness}. 
Furthermore, approaches for mitigating discrimination bias can be generally categorized into (i) pre-processing methods, which modify the distribution of the training data, see \cite{kamiran2012data,hajian2013methodology,feldman2015certifying,calmon2017optimized}; (ii) in-processing methods, which modify the cost function or the constraints of the learning algorithm, see \cite{calders2010three,fish2016confidence,kamishima2011fairness,zafar2017fairness}; (iii) post-processing methods, which modify the prediction outcome, see \cite{hardt2016equality,pedreschi2009measuring}, and finally, (iv) causal reasoning, which introduces the concepts of counterfactual and interventional fairness, see \cite{kilbertus2017avoiding,kusner2017counterfactual, pearl2009causality}. This work adopts a group fairness approach, and classifies as a pre-processing method for bias mitigation. Specifically, we introduce a fairness-utility score for each (non-sensitive) feature, based on information theoretic measures, and use these scores for fairness-aware feature selection: Features with high discriminatory and low accuracy impact shall be removed. Further, we relate our proposed measure for the discriminatory impact of features to the existing notions of group fairness such as statistical parity and equalized odds (see Section \ref{sec:discrimination_effect}). Finally, we provide causal reasoning for the desired properties of the measures we construct for the accuracy and discriminatory impacts of features.\\ \\
{\bf{Fairness-aware feature selection.}} Feature selection under fairness constraints has been previously studied by a few works. \cite{grgic2016case} proposed a feature selection criteria based on the populations' consensus on whether features are unfair to use, and investigated the impact of removing these features on prediction accuracy. In \cite{kazemi2018scalable}, the authors considered a submodular robust optimization problem, in which a set function is maximized under a set-size constraint, despite deletions in data points. The authors examined their proposed algorithms on a feature selection setting where the sensitive features (protected attributes) are deleted and the objective function is the mutual information between subsets of features and the true label. In \cite{ghassami2018fairness}, the authors proposed an optimization framework where the objective is to find the optimal compressed version of the features which maximizes the mutual information with the true label, while their mutual information with the protected attribute, given the true label, is upper bounded. In \cite{galhotra2020fair}, the authors considered the problem of integrating new features to an existing training dataset, and proposed an algorithm to identify the new features that can be added to the original training data while still ensuring interventional fairness.\\ \\
{\bf{Transparency and Shapely value function.}} There is a growing interest in transparency and interpretability analysis of machine learning systems \cite{datta2016algorithmic, doshi2017towards,goodman2017european}. In \cite{datta2016algorithmic}, the authors proposed a framework for transparency analysis which is easy to interpret for humans. This is achieved by quantifying the individual impacts of features on various metrics of a decision maker, such as decision error, risk, or discrimination. Our measures for accuracy and discriminatory impacts can also be viewed as measures for transparency. Besides, there has been a surge in using Shapely value for transparent machine learning, i.e., quantifying the importance of features for accurate and/or fair decisions of a given classifier \cite{datta2016algorithmic, lundberg2017unified, mase2021cohort}. In addition, a few work proposed using Shapely value for feature selection, without fairness considerations, see \cite{cohen2005feature,sun2012using}. In this work, we also use Shapely value function to deduce marginal impacts of features, while accounting for their interdependencies. However, our main contribution in this work lies in constructing {\it{information theoretic}} measures for the accuracy and discriminatory impacts of features, which do not depend on a certain choice of the classifier.\\ \\
{\bf{Information theoretic measures for fairness.}} Another related work is
\cite{dutta2020information} which proposed using information theoretic tools to measure the non-exempt and exempt components of discrimination: The non-exempt discrimination component quantifies the part of discrimination that cannot be accounted for by the features critical for accuracy, and the exempt component quantifies the remaining part of discrimination. In \cite{alghamdi2020model}, the authors proposed an information theoretic formulation for ``model projection'', where a reference probablistic classifier is projected to the set of classifiers that satisfy certain fairness criteria.

\subsection{Challenges and our contributions}
Our goal is to develop a framework for fairness-aware feature selection by computing a score for each feature which captures both its accuracy and discriminatory impacts. Thus, we need to define measures which (i) precisely quantify the accuracy and discriminatory impacts of a feature (or a subset of features), and (ii) capture the inter-dependencies among the features, the protected attribute, and the true outcome. These measures depend only on the joint statistic of the data and not on the particular classifier at hand. 

We tackle these challenges by proposing information-theoretic accuracy and discrimination measures that depend only on the joint distribution of the data. In order to account for the {\it{correlation}} among the model variables, these measures are defined on subsets of features rather than a single feature. As our ultimate goal is to estimate the marginal impact of each feature, we propose extracting a score for each feature using the so-called Shapley value \cite{shapley1953value}, a concept from cooperative game theory that allows assigning values to quantify an individual's contribution to the game. Our measures are based on a bivariate decomposition of mutual information \cite{bertschinger2014quantifying}, in order to precisely quantify the accuracy and discriminatory impacts.

The remainder of this paper is organized as follows. Section \ref{sec:accdesmeasures} introduces accuracy and discrimination measures, defined for subsets of features. In Section \ref{sec:agg}, we compute the aggregate discriminatory and accuracy impact of each feature. Section \ref{sec:exp} discusses experimental results. All the proofs are deferred to the appendix. 

\section{Accuracy and discrimination measures}\label{sec:accdesmeasures}
As mentioned earlier, we seek information-theoretic measures. But generic measures such as mutual information are not suitable for estimating the accuracy and discriminatory effects of a subset of features. We clarify this in Sections \ref{sec:accuracy_effect} and \ref{sec:discrimination_effect}. Instead, we propose information-theoretic measures based on a bivariate decomposition of mutual information \cite{bertschinger2014quantifying}, which are indeed able to meet our requirements.

\subsection{Bivariate decomposition of information}\label{sec:bivariate}

\begin{figure}
    \centering
    \includegraphics[width=0.45\linewidth]{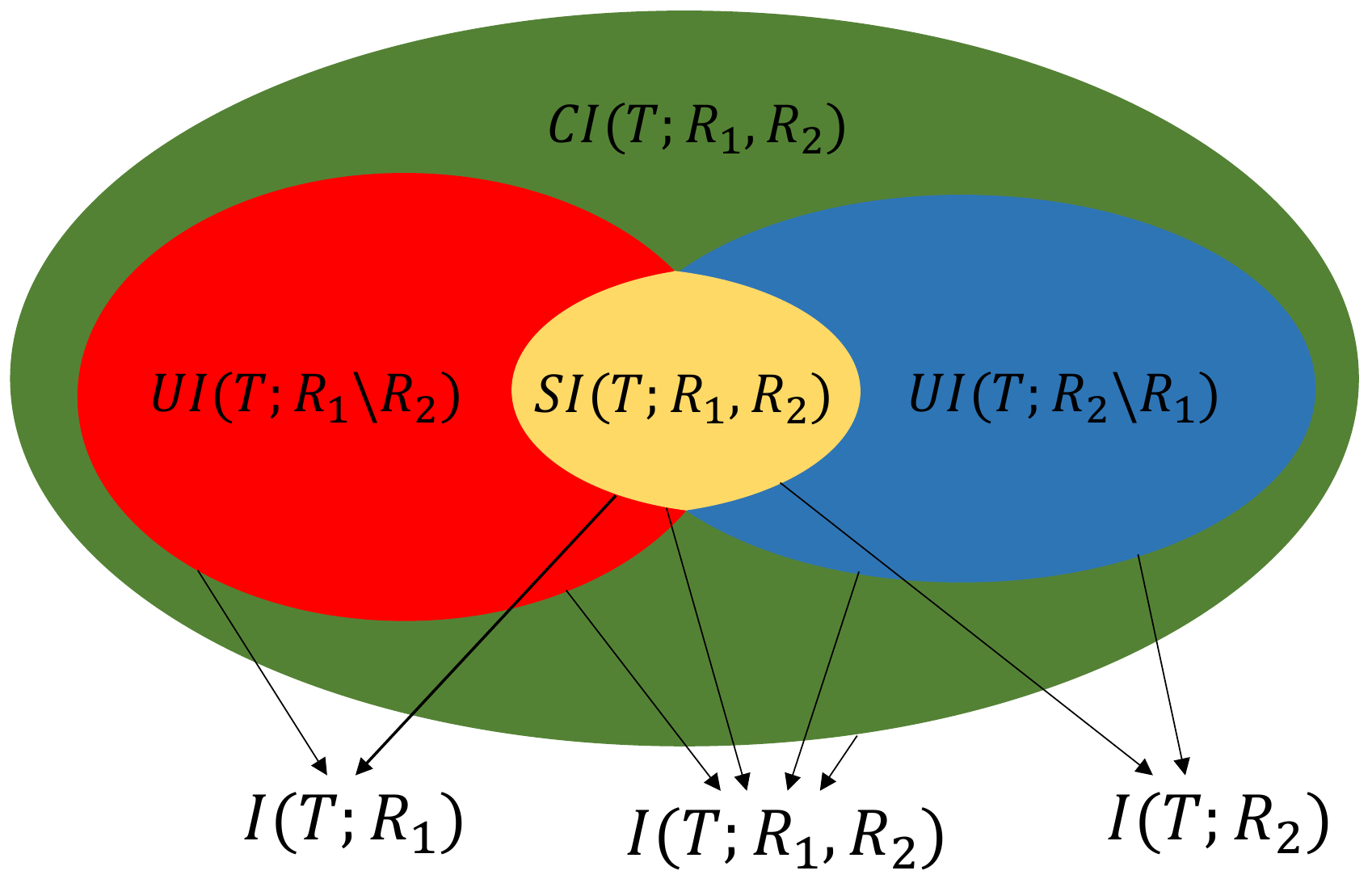}
    \caption{Decomposition of Information.}
    \label{fig:info_decompose}
\end{figure}

Consider an arbitrary triple of random variables $R_1$, $R_2$, and $T$, and let $P_{T,R_1,R_2}$ denote their joint probability distribution. The amount of information $R_1$ and $R_2$ contain about $T$ is measured by the mutual information $I(T;R_1,R_2)$. \cite{bertschinger2014quantifying} proposed a non-negative decomposition of $I(T;R_1,R_2)$ into unique, shared, and synergistic information; see Figure \ref{fig:info_decompose}. The unique information of $R_1$ with respect to $T$, denoted by $UI(T;R_1\setminus R_2)$, represents the information content related to $T$ that is only available in $R_1$. The shared information of $R_1$ and $R_2$, denoted by $SI(T;R_1,R_2)$, represents the information content related to $T$ that both $R_1$ and $R_2$ possess. Finally, the synergistic information of $R_1$ and $R_2$, denoted by $CI(T;R_1,R_2)$, represents the information content that can be obtained only if both $R_1$ and $R_2$ are available. For instance, in a cryptographic system, both the secret key and the cypher text are needed to decode the encrypted message. 
According to Figure \ref{fig:info_decompose}, this decomposition of mutual information satisfies the following:
\begin{align}
&I(T;R_1,R_2) = UI(T;R_1\backslash R_2)+UI(T;R_2\backslash R_1) + SI(T;R_1,R_2) + CI(T;R_1,R_2),\label{eq:I}\\
&I(T;R_i) = UI(T;R_i\backslash R_j)  + SI(T;R_1,R_2), ~ i\neq j, ~i,j\in\{1,2\}\label{eq:I_1}. 
\end{align}
By defining a measure for either unique, shared, or synergistic information, and using the decomposition in (\ref{eq:I}), (\ref{eq:I_1}), the other two quantities are well defined. We use the following measure for unique information  \cite{bertschinger2014quantifying}:
\begin{equation}\label{eq:UI}
UI(T;R_1\backslash R_2) = \min_{Q\in \Delta_P} I_Q(T;R_1|R_2),
\end{equation}
where $I_Q(T;R_1|R_2) = \sum_{t,r_1,r_2} Q_{T,R_1,R_2}(t,r_1,r_2)\log\frac{Q_{T|R_1,R_2}(t|r_1,r_2)}{Q_{T|R_2}(t|r_2)}$ is the conditional mutual information between $T$ and $R_1$ given $R_2$, calculated with respect to the joint probability distribution $Q\in \Delta_P$.  $\Delta_P=\{Q\in\Delta: Q_{T,R_1} = P_{T,R_1}, Q_{T,R_2} = P_{T,R_2} \}$, where $\Delta$ is the simplex of joint probability distributions over $T$, $R_1$, $R_2$. 

Using the unique information definition in \eqref{eq:UI}, and the information decomposition in \eqref{eq:I}--\eqref{eq:I_1}, we can uniquely define the shared and synergistic information. In Sections \ref{sec:accuracy_effect}, \ref{sec:discrimination_effect}, we use the information decomposition in \eqref{eq:I}--\eqref{eq:UI} to construct information-theoretic measures for the accuracy and discriminatory impacts of subsets of features.

\subsection{Problem formulation} 
We consider a supervised learning setting in which each individual in the dataset is associated with the protected attribute(s) $A\in\mathcal{A}$, and a set of $n$ features $X^n=\{X_1,\cdots, X_n\}$, where $X_i\in\mathcal{X}$, $i\in[n]=\{1,2,\cdots,n\}$. For $S\subseteq [n]$, $X_S$ is the subset of features $\{X_i: i\in S\}$, and $X_{S^c}= X^n\setminus X_S$. For the classification task, let $Y\in\mathcal{Y}$ and $\Hat{Y}\in\mathcal{Y}$ be the true label and the predicted label of an individual, respectively.

To demonstrate the desired properties for our measures, we use a causal graph to represent the relations among the system variables \cite{pearl2009causality}. A causal graph is a directed acyclic graph whose nodes are random variables. A directed edge from $X$ to $Y$ indicates that $X$ is a direct cause of $Y$; we call $X$ a \textit{parent} of $Y$, and $Y$ a \textit{child} of $X$. If there is a directed path from $X$ to $Y$, $Y$ is called a \textit{descendant} of $X$, and $X$ an \textit{ancestor} of $Y$. For a joint probability distribution $P$ on the variables in a causal graph $G$, we assume that $P$ satisfies both Markov and faithfulness conditions with respect to $G$ \cite{pearl2009causality}. Thus, conditional independencies among the variables can be read from the graph \cite{pearl2009causality}. We require faithfulness since the desired properties for our measures are defined on the corresponding causal graph.  

For simplicity of demonstrating our measures, we assume that the protected attributes $A$ influence the true label $Y$ only through the features $X^n$. That is, $A$ does not have a direct causal influence on $Y$. As in \cite{kusner2017counterfactual}, we also assume ancestral closure of $A$, i.e., any parent of $A$ should be in $A$. For example, if race is a protected attribute, and mother's race is a parent of race, then mother's race should be a protected attribute as well. The following graphical model governs the statistical relations among the variables:
\begin{equation}
\label{eq:GM}
\begin{aligned}
	A\rightarrow &\hspace{0.5mm} X^n \rightarrow {Y}\\[-5pt]
	&\downarrow\\[-4pt]
	&\hspace{0.5mm}\Hat{Y}
\end{aligned}
\end{equation}
where $\Hat{Y}=f(X^n)$. $f: \mathcal{X}^n \mapsto \mathcal{Y}$ represents the classifier, where $\mathcal{X}^n$ is the $n$th Cartesian product of $\mathcal{X}$. 
\subsection{Quantifying accuracy effect}\label{sec:accuracy_effect}
Suppose $X=\{X_1,X_2\}$. To measure the impact of $\{X_1\}$ in accurately predicting $Y$, one candidate measure is $I(X_1;Y)$. However, this quantity contains the shared information between $X_1$ and $X_2$, with respect to $Y$, i.e., $SI(Y; X_1$, $X_2$), which should not be attributed to $X_1$ alone. For example, in the extreme case where $X_1$ and $X_2$ are copies of one another, i.e., $X_1=X_2$, it follows that $X_1$ is not essential for prediction. However, $I(X_1;Y)$ is equal to the total content of information. This suggests omitting the shared information from $I(X_1;Y)$. A second candidate measure is the unique information of $X_1$, i.e.,  $UI(Y;X_1\setminus X_2)$. This measure is not adequate either. For example, let $X_1,X_2 \in \{0,1\}$ be such that $P(X_1=1) =P(X_2=1) = 1/2+\epsilon$, and let $Y= X_1 \oplus X_2$. As $\epsilon \rightarrow 0$, $I(X_1;Y)\rightarrow 0$, which implies that $UI(Y;X_1\backslash X_2)\rightarrow 0$, while $X_1$ clearly impacts the accuracy of prediction. 

We postulate that a good accuracy measure for a subset of features $X_S\subseteq X^n$, denoted by $v^{Acc}(X_S)$, should satisfy the following properties:
\begin{itemize}\label{axiom}
    \item Non-negativity: $v^{Acc}(X_S) \geq 0$,  $\;\forall X_S\subseteq X^n$.
    \item Monotinicity: $v^{Acc}(X_{S_1}) \leq v^{Acc}(X_{S_2})$,  $\;\forall S_1 \subseteq S_2$.
    \item Blocking: $Y\independent X_S|\{A,X_{S^c}\} \iff v^{Acc}(X_S) = 0$. 
\end{itemize}
Including an additional feature should not decrease prediction accuracy, hence $v^{Acc}(X_S)$ should satisfy monotinicity. By assigning a zero measure to the empty set ($v^{Acc}(\emptyset)=0$), $v^{Acc}(X_S)$ is non-negative. The blocking property implies that, in the causal graph representing the relations among the variables, $v^{Acc}(X_S)$ is non-zero for features that form a {\it{Markov blanket}} of ${Y}$, and is zero for the remaining features. The two measures mentioned earlier (mutual and unique information) do not satisfy these properties. Mutual information does not satisfy blocking, and unique information is not monotone in $X_S$ \cite{bertschinger2014quantifying}.

We propose to use the sum of unique information (that is only available in $X_S$) and synergistic information of $X_S$ and all other variables, as a measure of the accuracy impact of $X_S$. This is equivalent to the conditional mutual information $I(Y;X_S| X_{S^c})$, which is the content of information that would be lost if we eliminate $\{X_S\}$ from the set $X^n$. 

\begin{definition}[Accuracy coefficient]
\label{def:1}
For a subset of features $X_S\subseteq X^n$, the accuracy coefficient of $X_S$ is given by
\begin{align}
\label{eq:vAcc}
 &v^{Acc}(X_S) = I(Y; X_S| \{A, X_{S^c}\})= UI(Y;X_S\backslash \{A,X_{S^c}\}) +
CI(Y;X_S, \{A,X_{S^c}\}).
\end{align}
\end{definition}
$v^{Acc}(X_S)$ satisfies all the desired properties. Non-negativity and monotinicity are straightforward. Theorem \ref{thm:1} below states that $v^{Acc}(X_S)$ satisfies the blocking property.

\begin{theorem} \label{thm:1}
For any subset of features $X_S\subseteq X^n$, the accuracy coefficient $v^{Acc}(X_S)$ is zero if and only if $X_S$ is not a direct cause of ${Y}$ in the corresponding causal graph, i.e., $Y\independent X_S | X_{S^c}$.
\end{theorem}
\begin{proof}
The proof of Theorem \ref{thm:1} appears in Appendix \ref{appendix_thm1}.
\end{proof}

\begin{remark}
In Definition \ref{def:1}, we included the protected attribute $A$ along with the other features $X_{S^c}$ in the conditioning (and in the unique and synergistic information). This follows due to the assumption that the protected attribute $A$ is not a direct cause of $Y$, and hence $A\not\subset X^n$. If $A$ were to be a subset of $X^n$, then $\{A,X_{S^c}\}$ should be replaced by $X_{S^c}$ in \eqref{eq:vAcc}.
\end{remark}

\subsection{Quantifying discriminatory effect}
\label{sec:discrimination_effect}
As we mentioned earlier, discrimination can be explicit when the protected attribute $A$ is used as an input to the classifier, or implicit when the protected attribute influences the outcome through proxy features. Our goal in this section is to measure the discriminatory impact of a subset of features $X_S\subseteq X^n$. A measure for the discriminatory impact of $X_S$, denoted by $v^{D}(X_S)$, should satisfy the following properties:
\begin{itemize}\label{axio}
    \item Non-negativity: $v^D(X_S) \geq 0$
    \item Monotinicity: $v^D(X_{S_1}) \leq v^D(X_{S_2})~ \text{ for } ~ S_1 \subseteq S_2$
    \item Y-independence: $Y\independent X_S  \implies v^D(X_S) = 0$.
    \item A-independence: $A\independent X_S  \implies v^D(X_S) = 0$.
    \item AY-independence: $A\independent X_S|Y  \implies v^D(X_S) = 0$.
\end{itemize}
Non-negativity and monotonicity have the same reasoning as for the accuracy measure in Section \ref{sec:accuracy_effect}. We justify Y-independence and A-independence as follows. A subset $X_S$ that is ``irrelevant" to the classification task ($Y\independent X_S$), or does not act as a proxy for the protected attributes ($A\independent X_S$), should not be considered as discriminatory, i.e., $v^D(X_S)=0$. Similarly, AY-independence implies that, conditioned on the class label $\{Y=y\}$, if $X_S$ is independent of $A$, $v^D(X_S)$ should be zero. Note that neither $I(X_S;A)$ nor $I(X_S;A|Y)$ satisfies our desired properties. $I(X_S;A)$ does not satisfy Y-independence and AY-independence, and $I(X_S;A|Y)$ does not satisfy Y-independence and A-independence. Instead, we propose the following discrimination measure:

\begin{definition}[Discrimination coefficient]\label{def:2}
For a subset of features $X_S\subseteq X^n$, the discrimination coefficient is 
\begin{align}
\label{eq:vD}
v^{D}(X_S) \triangleq SI(Y;X_S,A) \times I(X_S;A) \times I(X_S;A|Y). 
\end{align}
\end{definition}

$SI(Y;X_S,A)$ measures the information content in $Y$ provided by $X_S$, and is ``discriminatory'' in the sense that it is shared with the protected attribute $A$. $I(X_S;A)$ measures the dependence between $X_S$ and $A$. $I(X_S;A|Y)$ measures the conditional dependence between $X_S$ and $A$, given $Y$. Our discrimination measure satisfies all the aforementioned properties:

\begin{proposition}\label{prop:1}
The discrimination coefficient $v^D(X_S)$ is monotone and non-negative. That is, (i) For any $S_1 \subseteq S_2$, we have  $v^D(X_{S_1}) \leq v^D(X_{S_2})$, and (ii) $v^D(X_S) \geq 0$, for all $X_S\subseteq X^n$.
\end{proposition}

\begin{proof}
The proof of Proposition \ref{prop:1} is given in Appendix \ref{appendix_prop1}.
\end{proof}

\begin{proposition}\label{prop:2}
If a subset of features $X_S\subseteq X^n$ satisfies $X_S\independent Y$, $X_S\independent A$, or $X_S\independent A|Y$, then $v^D(X_S) = 0$.
\end{proposition}
\begin{proof}
The proof of Proposition \ref{prop:2} appears in Appendix \ref{appendix_prop2}.
\end{proof}

Besides satisfying the desired properties, our discrimination measure $v^D(X_S)$ is in line with the notions of fairness in the literature \cite{dwork2012fairness, zemel2013learning,hardt2016equality}. 

\begin{definition}[Notions of Fairness]
\label{def:0}
A classifier $\hat{Y}:\mathcal{X}^n\rightarrow\mathcal{Y}$ satisfies demographic parity if it is statistically independent of the protected attribute ($\hat{Y}\independent A$). Further, $\Hat{Y}$ satisfies equalized odds if it is statistically independent of the protected attribute conditioned on the true label  ($\hat{Y}\independent A|Y$).
\end{definition}

To restrict our measures in \eqref{eq:vAcc} and \eqref{eq:vD} to a certain classifier $f(X^n)$, we can replace $X_S$ by $\Hat{Y}=\left.f\right|_{X_S}(X^n)$, where $\left.f\right|_{X_S}(X^n) = f(X^n)$ such that $X_{S^c}$ is held constant. For the discrimination measure $v^D(X_S)$, by replacing $X_S$ with $\Hat{Y}$, A-independence and AY-independence in Proposition \ref{prop:2} correspond to demographic parity and equalized odds.

We further investigate the properties of $v^D(X_S)$. In the following definition, a direct path refers to a chain of variables connected with directed arrows in the graphical model.

\begin{definition}
A subset of features $X_S\subseteq X^n$ is called path-discriminatory if it blocks all the direct paths from the protected attribute $A$ to the true label $Y$, i.e., $A\independent Y|X_S$.
\end{definition}

\begin{theorem}\label{thm: 2}
For every subset of features $X_S$ that is path-discriminatory, $v^D(X_S) = I(Y;A)I(X_S;A)I(X_S;A|Y)$. Further, for a path-discriminatory subset of features $K$ that is directly connected to the protected attribute $A$, i.e.,  $A\independent \left\{Y \cup \{X^n\setminus K\}\right\}|K$, we have $v^D(K) \geq v^D(X_S)$, for all $X_S\subseteq X^n$.
\end{theorem}

\begin{proof}
The proof of Theorem \ref{thm: 2} is given in Appendix \ref{appendix_thm2}.
\end{proof}

\section{Aggregation of Effects}\label{sec:agg}
In Section \ref{sec:accdesmeasures}, we defined accuracy and discrimination measures for subsets of features, $v^{Acc}(X_S)$ and $v^D(X_S)$. However, we shall not use $v^{Acc}(X_i)$ and $v^D(X_i)$ to measure the \textit{marginal} accuracy and discrimination impacts of a single feature $X_i$; these do not take into account the correlation among features. For instance, from the definition of shared information, it is evident that $SI(Y;X_i,A)$ is a function of only the marginals $P_{YX_i}$ and $P_{YA}$. 
In order to account for this correlation, we need to factor in the aggregate effect of all subsets of features that include a certain feature. For example, to measure the contribution of a single State's electoral votes on winning the election, we need to first measure the impact of that state when included with all possible combinations of other states, and aggregate these to deduce its marginal impact. We can calculate the effect of adding $X_i$ to an arbitrary subset $X_S$ by $v(X_S\cup \{X_i\}) - v(X_S)$. An appropriately weighted sum of these effects can provide an aggregation measure. This leads us to the \textit{Shapley value} function \cite{shapley1953value}:

\begin{definition}
Let $\mathcal{P}$ denote the power set. Given a \textit{characteristic function} $v(\cdot): \mathcal{P}([n])\rightarrow \mathbb{R}$, the Shapley value function $\phi_{(\cdot)}: [n]\rightarrow \mathbb{R}$ is defined as:
\[
\phi_i = \sum_{T\subseteq [n]\backslash i} \frac{|T|!(n-|T|-1)!}{n!} (v(T\cup \{i\}) - v(T)), ~\forall i\in[n].
\]
Given the characteristic functions  $v^{Acc}(\cdot)$ and $v^D(\cdot)$, the corresponding Shapley value functions are denoted by $\phi^{Acc}_{(\cdot)}$ and $\phi^D_{(\cdot)}$. We refer to these as \textit{marginal accuracy coefficient} and \textit{marginal discrimination coefficient}, respectively.
\end{definition}
The weights $\frac{|T|!(n-|T|-1)!}{n!}$ in the definition of the Shapley value function are chosen so that the following lemma holds:
\begin{lemma}\label{lem:44}
\cite{young1985monotonic} Shapley value is the unique aggregation function satisfying the following properties:
\begin{itemize}
    \item Symmetry: If $v(T\cup \{i\}) = v(T\cup \{j\})$, for all $T\subseteq [n]\setminus \{i,j\}, \implies \phi_i = \phi_j$.
    \item Efficiency: $\sum_{i\in [n]} \phi_i = v([n])$.
    \item Monotonicity: Given two characteristic functions $v^{(1)}(\cdot)$, $v^{(2)}(\cdot)$, and the corresponding Shapley value functions $\phi_{(\cdot)}^{(1)}$, $\phi_{(\cdot)}^{(2)}$, if $v^{(1)}(T\cup \{i\}) - v^{(1)}(T) \geq v^{(2)}(T\cup \{i\}) - v^{(2)}(T), \forall T\subseteq [n] \implies \phi_i^{(1)} \geq \phi_i^{(2)}$.
\end{itemize}
\end{lemma}
Besides these properties, we have the following result:
\begin{corollary}\label{cor:1}
Marginal accuracy and discrimination coefficients are non-negative, i.e., $\phi_i^{Acc}$, $\phi_i^D \geq 0$, for all $i\in [n]$.
\end{corollary}
\begin{proof}
The proof of Corollary \ref{cor:1} appears in Appendix \ref{appendix_cor1}.
\end{proof}

The following corollary captures the special case when one of the features, say $X_a$, where $a\in[n]$, is a sufficient statistics for the true label $Y$. In this case, $X_a$ should attain the highest marginal accuracy coefficient among all features. 
\begin{corollary}\label{cor:acc}
In the special case where $Y$ has a single parent $X_a$, i.e., $Y\independent \{A\cup \{X^n \setminus X_a\}\}|X_a$, we have $\phi_a^{Acc} \geq \phi_j^{Acc}$, for all $j\in [n]\setminus a$.
\end{corollary}
\begin{proof}
The proof for Corollary \ref{cor:acc} appears in Appendix \ref{appendix_cor2}.
\end{proof}

Corollary \ref{cor:allayp} below states, if a certain feature is a sufficient statistics for the protected attribute, its marginal discrimination coefficient should be the highest among all features. 

\begin{corollary}\label{cor:allayp}
In the special case where $A$ has a single child $X_d$, and $A\independent (Y\cup \{X^n \setminus X_d\})|X_d$;  $\phi_d^D \geq \phi_j^D, \forall j\in [n]\setminus d$.
\end{corollary}
\begin{proof}
The proof of Corollary \ref{cor:allayp} is provided in Appendix \ref{appendix_cor3}.
\end{proof}
A feature that is a single child of $A$ is the ``worst'' feature in terms of discrimination. This feature is the most correlated with $A$, and all the information that pertains to $A$ passes through this feature to the other nodes of the graph. In other words, removing this feature results in $A$ and $Y$ being independent.

We can use the marginal coefficients $\phi^{Acc}_i$ and $\phi^D_i$ to define a score for each feature (for feature selection), termed as the fairness-utility score. Specifically, given $\phi^{Acc}_i$ and $\phi^D_i$, the fairness-utility score for a feature $X_i$ is defined as $\mathcal{F}_i = \phi^{Acc}_i - \alpha \phi^D_i$, where $\alpha$ is a positive hyperparametter which trades off between accuracy and discrimination.
\section{Experimental Results} \label{sec:exp}
We evaluate our accuracy and discrimination measures both on real and synthetic datasets.

\subsection{Synthetic Dataset}
We evaluate our measures on a randomly generated dataset, following the graphical model in Figure \ref{fig:graphsyn}. We randomly assigned a conditional distribution for each variable given their parents in the graph. We calculated the joint distribution $P_{XAY}$ from these conditional probabilities. We calculated the marginal accuracy and discrimination coefficients for this joint distribution. The results are shown in Table \ref{tab:syn}. 

Furthermore, we randomly split the whole dataset into training and test subsets. We trained a neural network classifier with the features as the input, and prediction of $Y$ as the output. We removed one feature at a time and repeated the same procedure with the remaining features. Similar to the real dataset, the prediction error is calculated using cross entropy loss, and the bias is calculated by the KL divergence between $P_{\hat{Y}|A=0}$ and $P_{\hat{Y}|A=1}$. We ran this procedure 100 times, and calculated the average and the confidence intervals. The results are shown in Figure \ref{fig:syndrop}. It is clear that removal of $X_3$ and $X_4$, which are the most important features according to the marginal accuracy coefficient, results in a higher increase in the prediction error. In addition, as suggested by the marginal discrimination coefficient, $X_1$ is the most discriminatory feature. We observe that removing $X_1$ results in the lowest bias in the prediction.

\begin{table}
\centering
\begin{tabular}{l*{6}{c}r}
Features      & $^{\scriptscriptstyle 10^2\times}\phi^{Acc}$ &  $^{\scriptscriptstyle 10^7\times}\phi^{D}$ \\
\thickhline
1) \texttt{$X_1$} & $0.027$ & $2.7927$\\
\hline
2) \texttt{$X_2$}  & $0.052$ & $0.0923$\\
\hline
3) \texttt{$X_3$}       & $4.467$ & $0.0959$\\
\hline
4) \texttt{$X_4$}     & $3.122$ & $0.0899$\\
\hline
5) \texttt{$X_5$}     & $0.141$ & $1.4301$\\
\end{tabular}
\caption{Discrimination and accuracy coefficients For synthetic dataset.}
\label{tab:syn}
\end{table}
\begin{figure}
\centering
\begin{subfigure}{.3\textwidth}
  \centering
  \includegraphics[width=1\linewidth]{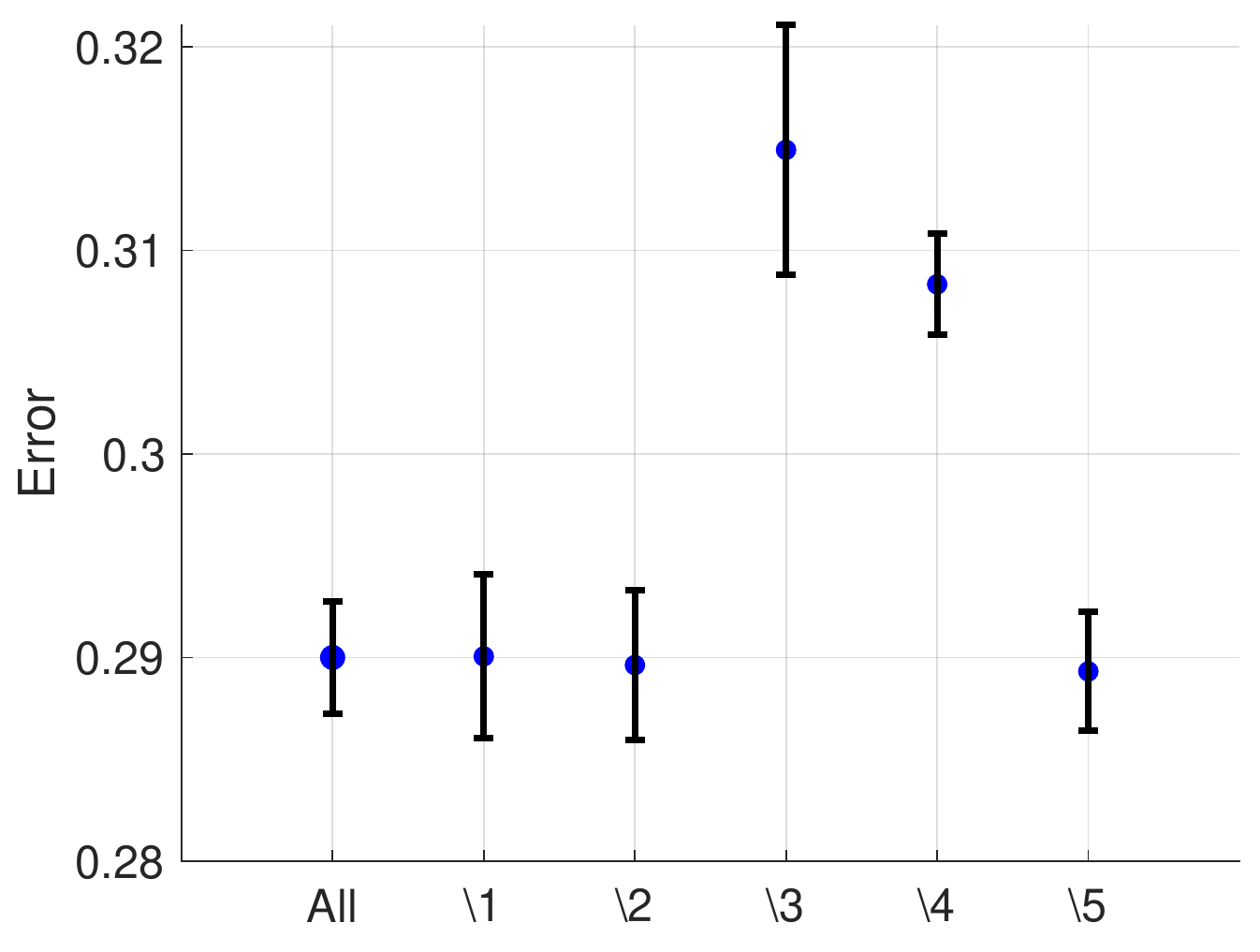}
  \label{fig:sub11}
\end{subfigure}
\begin{subfigure}{.3\textwidth}
  \centering
  \includegraphics[width=1\linewidth]{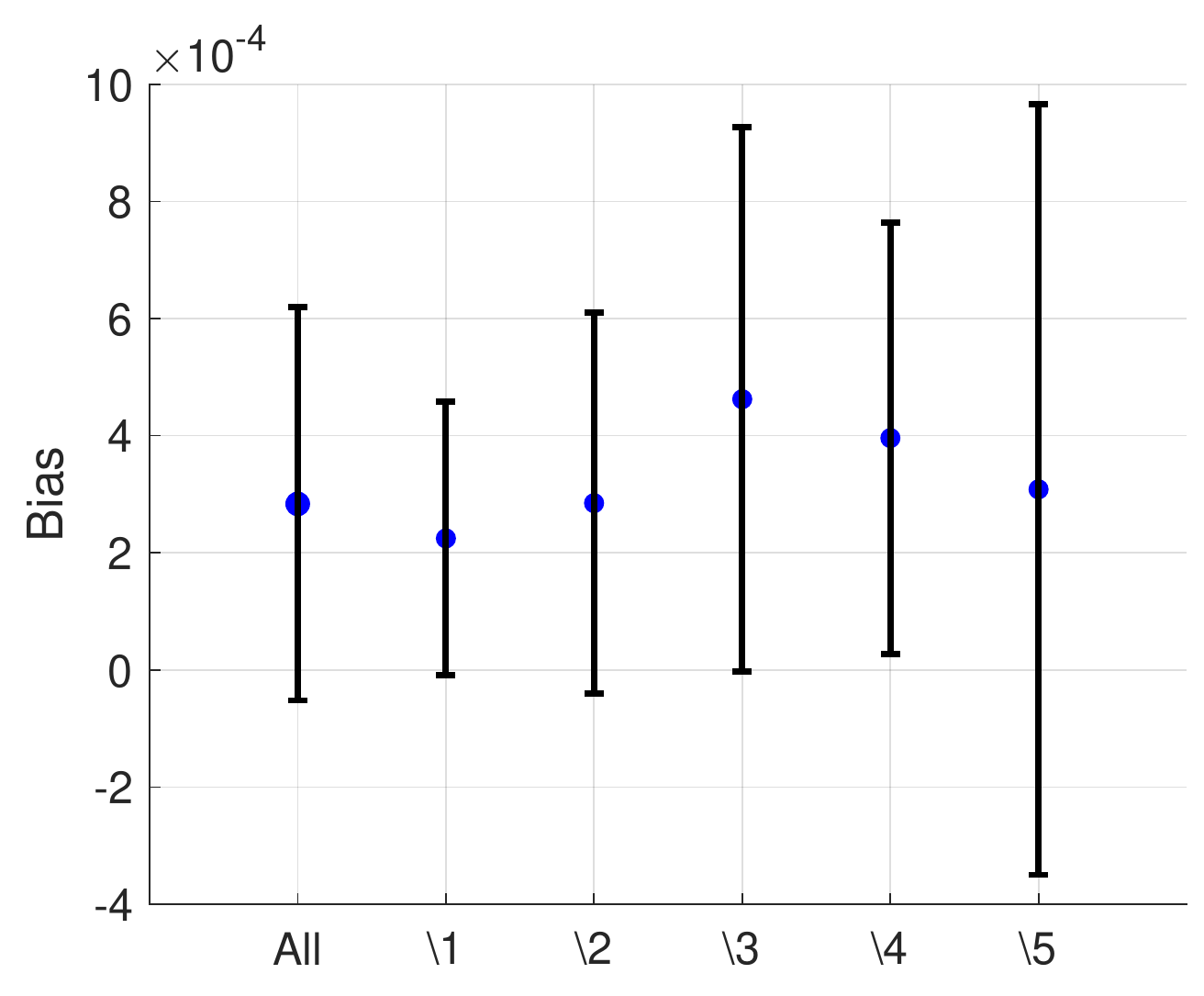}
  \label{fig:sub22}
\end{subfigure}
\caption{a) Error and b) Bias of the classifier over synthetic dataset when all the features are used, and when one of the features is removed.}
\label{fig:syndrop}
\end{figure}

\begin{figure}
    \centering
    \includegraphics[width=0.35\linewidth]{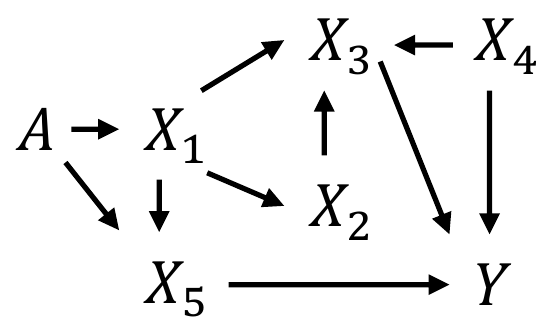}
    \caption{The graphical model of the data generating process of the synthetic data.}
    \label{fig:graphsyn}
\end{figure}

\subsection{ProPublica COMPAS Dataset}
ProPublica COMPAS dataset \cite{angwin2016machine} contains the criminal history and demographic makeup of defendants in Broward County, Florida from 2013-2014. We processed the raw dataset by dropping records with missing information and converted categorical variables to numerical values. The race of each individual is provided in the COMPAS dataset and constitutes our protected attribute. We restrict our analysis to individuals who are African American (A = 0) or Caucasian (A = 1). Our processed dataset contains 5334 records (3247 African Americans and 2087 Caucasians). Each individual in our dataset has a feature vector (\texttt{Age}, \texttt{Charge Degree}, \texttt{Gender}, \texttt{Prior Counts}, \texttt{Length Of Stay}), and a binary true label which indicates whether the individual was arrested for a crime within 2 years of release. \texttt{Age} variable takes three levels of \texttt{Age} $<25$, $25<$ \texttt{Age} $<45$, or \texttt{Age} $ > 45$. \texttt{Charge Degree} has two values Misdemeanor or Felony, \texttt{Gender} is either Male or Female, \texttt{Prior Counts} can be $0$, $1-3$, or larger than $3$, and \texttt{Length of Stay} can be $\leq 1$  week, $\leq 3$  months, or $> 3$  months.

We applied our measures of accuracy and discrimination to this dataset (see Table \ref{tab:1}). The result shows that \texttt{Prior Count} and \texttt{Age} exhibit strongest proxies for discrimination. This is in line with the findings of \cite{wang2018direction}. In addition, the result shows that \texttt{Charge Degree} and \texttt{Gender} are the least informative features for the prediction task.

Furthermore, same as the synthetic dataset, we designed a classifier with all the features included, and with one of the features removed. The results are shown in Figure \ref{fig:compasdrop}. As it is predicted by the marginal discrimination coefficient, removal of \texttt{Age} or \texttt{Prior Counts} results in the lowest bias in the classifier output. In addition, as it is anticipated by the accuracy coefficient, removal of \texttt{Charge Degree} and \texttt{Gender} has very small effect on the error of the predictor.

\begin{table}
\centering
\begin{tabular}{l*{6}{c}r}
Features & $^{\scriptscriptstyle 10^2\times}\phi^{Acc}$ &  $^{\scriptscriptstyle 10^6\times}\phi^{D}$ \\
\thickhline
1) \texttt{Age} & $1.7892$ & $6.0904$\\
\hline
2) \texttt{Charge Degree}  & $0.1723$ & $1.3916$\\
\hline
3) \texttt{Gender}       & $0.4186$ & $1.6849$\\
\hline
4) \texttt{Prior Counts}     & $3.6116$ & $5.0466$\\
\hline
5) \texttt{Length of Stay}     & $1.8290$ & $2.3126$\\
\end{tabular}
\caption{Discrimination and accuracy coefficients for COMPAS dataset.}
\label{tab:1}
\end{table}

\begin{figure}
\centering
\begin{subfigure}{.3\textwidth}
  \centering
  \includegraphics[width=1\linewidth]{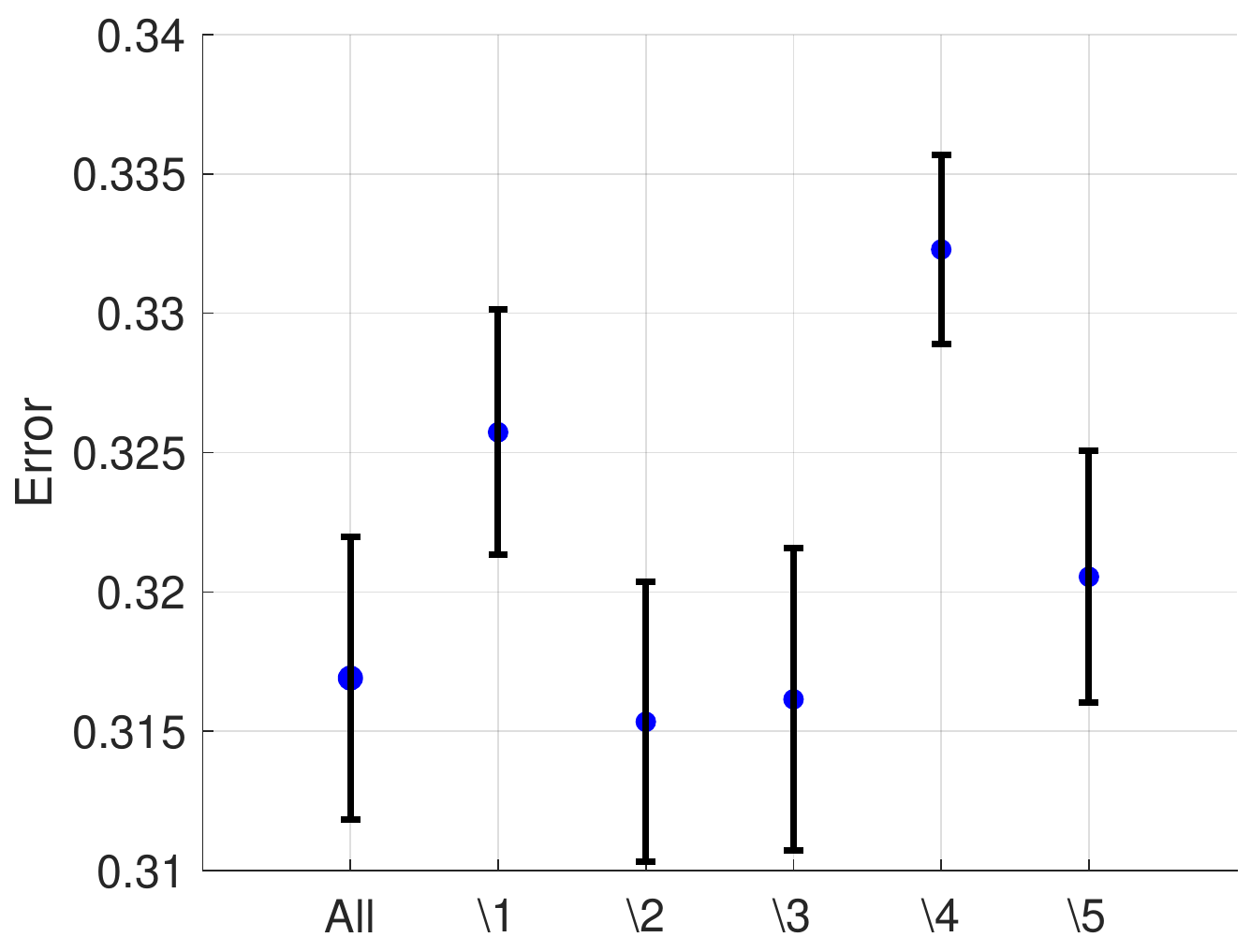}
  \label{fig:sub1}
\end{subfigure}
\begin{subfigure}{.3\textwidth}
  \centering
  \includegraphics[width=1\linewidth]{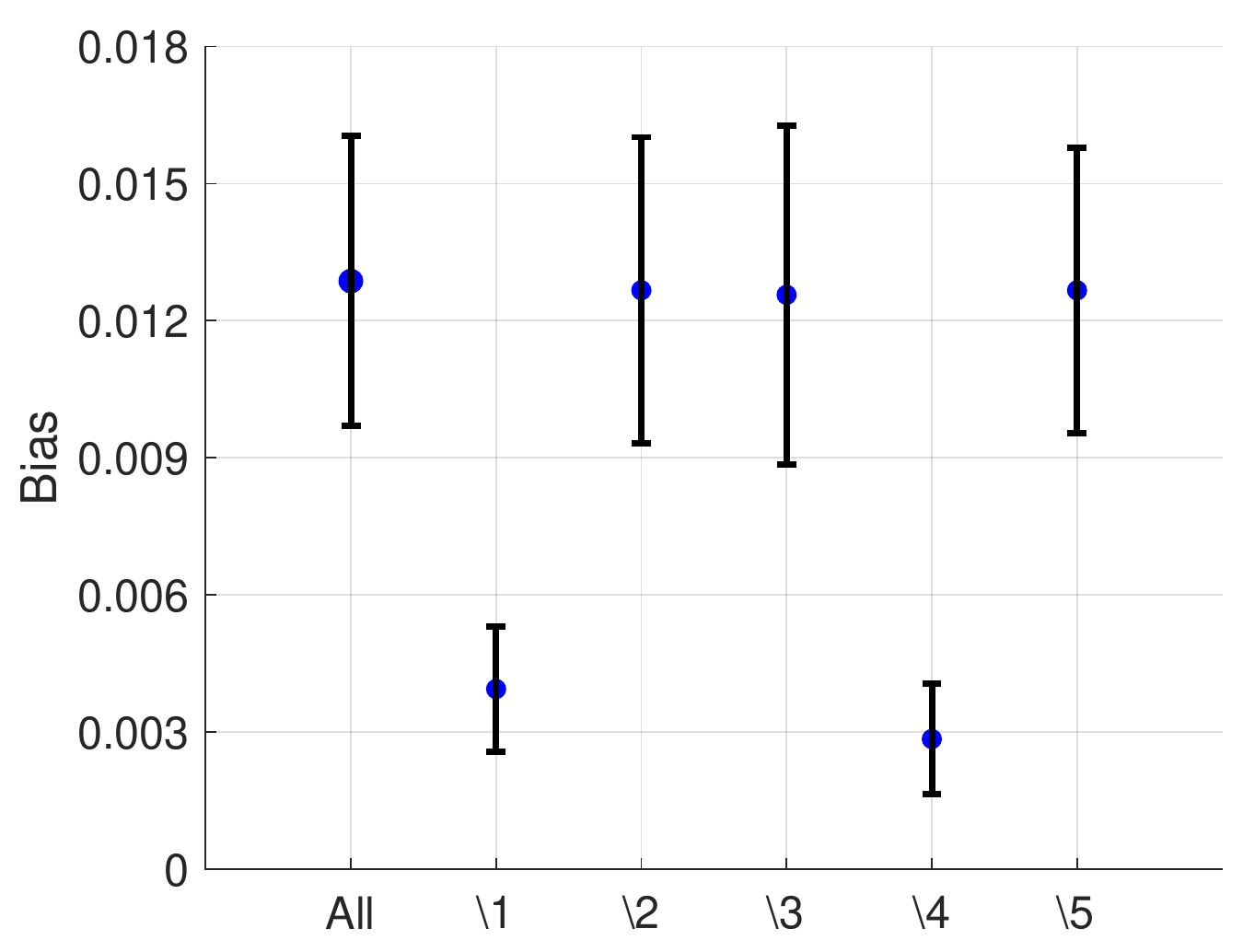}
  \label{fig:sub2}
\end{subfigure}
\caption{a) Error and b) Bias of the classifier over COMPAS dataset when all the features are used, and when one of the features is removed. $\backslash i$ denotes the case where $i$'th feature (as numerated in Table \ref{tab:1}) is removed.}
\label{fig:compasdrop}
\end{figure}

\newpage

\bibliographystyle{plain}   
\bibliography{Refs}

\begin{thebibliography}{10}

\bibitem{alghamdi2020model}
Wael Alghamdi, Shahab Asoodeh, Hao Wang, Flavio~P Calmon, Dennis Wei, and
  Karthikeyan~Natesan Ramamurthy.
\newblock Model projection: {T}heory and applications to fair machine learning.
\newblock In {\em 2020 IEEE International Symposium on Information Theory
  (ISIT)}, pages 2711--2716, 2020.

\bibitem{angwin2016machine}
Julia Angwin, Jeff Larson, Surya Mattu, and Lauren Kirchner.
\newblock Machine bias.
\newblock {\em Pro Publica}, 2016.

\bibitem{barocas2016big}
Solon Barocas and Andrew~D Selbst.
\newblock Big data's disparate impact.
\newblock {\em Cal. L. Rev.}, 104:671, 2016.

\bibitem{bertschinger2014quantifying}
Nils Bertschinger, Johannes Rauh, Eckehard Olbrich, J{\"u}rgen Jost, and Nihat
  Ay.
\newblock Quantifying unique information.
\newblock {\em Entropy}, 16(4):2161--2183, 2014.

\bibitem{calders2010three}
Toon Calders and Sicco Verwer.
\newblock Three naive bayes approaches for discrimination-free classification.
\newblock {\em Data Mining and Knowledge Discovery}, 21(2):277--292, 2010.

\bibitem{calmon2017optimized}
Flavio Calmon, Dennis Wei, Bhanukiran Vinzamuri, Karthikeyan~Natesan
  Ramamurthy, and Kush~R Varshney.
\newblock Optimized pre-processing for discrimination prevention.
\newblock In {\em Advances in Neural Information Processing Systems}, pages
  3992--4001, 2017.

\bibitem{cohen2005feature}
Shay Cohen, Eytan Ruppin, and Gideon Dror.
\newblock Feature selection based on the shapley value.
\newblock In {\em 19th international joint conference on Artificial
  intelligence}, pages 665--670. Morgan Kaufmann Publishers, 2005.

\bibitem{cover2012elements}
Thomas~M Cover and Joy~A Thomas.
\newblock {\em Elements of information theory}.
\newblock John Wiley \& Sons, 2012.

\bibitem{datta2016algorithmic}
Anupam Datta, Shayak Sen, and Yair Zick.
\newblock Algorithmic transparency via quantitative input influence: Theory and
  experiments with learning systems.
\newblock In {\em 2016 IEEE Symposium on Security and Privacy)}, pages
  598--617, 2016.

\bibitem{doshi2017towards}
Finale Doshi-Velez and Been Kim.
\newblock Towards a rigorous science of interpretable machine learning.
\newblock {\em arXiv preprint arXiv:1702.08608}, 2017.

\bibitem{dutta2020information}
Sanghamitra Dutta, Praveen Venkatesh, Piotr Mardziel, Anupam Datta, and Pulkit
  Grover.
\newblock An information-theoretic quantification of discrimination with exempt
  features.
\newblock In {\em Proceedings of the AAAI Conference on Artificial
  Intelligence}, volume~34, pages 3825--3833, 2020.

\bibitem{dwork2012fairness}
Cynthia Dwork, Moritz Hardt, Toniann Pitassi, Omer Reingold, and Richard Zemel.
\newblock Fairness through awareness.
\newblock In {\em Proceedings of the 3rd Innovations in Theoretical Computer
  Science Conference}, pages 214--226. ACM, 2012.

\bibitem{feldman2015certifying}
Michael Feldman, Sorelle~A Friedler, John Moeller, Carlos Scheidegger, and
  Suresh Venkatasubramanian.
\newblock Certifying and removing disparate impact.
\newblock In {\em Proceedings of the 21th ACM SIGKDD International Conference
  on Knowledge Discovery and Data Mining}, pages 259--268. ACM, 2015.

\bibitem{fish2016confidence}
Benjamin Fish, Jeremy Kun, and {\'A}d{\'a}m~D Lelkes.
\newblock A confidence-based approach for balancing fairness and accuracy.
\newblock In {\em Proceedings of the 2016 SIAM International Conference on Data
  Mining}, pages 144--152. SIAM, 2016.

\bibitem{galhotra2020fair}
Sainyam Galhotra, Karthikeyan Shanmugam, Prasanna Sattigeri, and Kush~R
  Varshney.
\newblock Fair data integration.
\newblock {\em arXiv preprint arXiv:2006.06053}, 2020.

\bibitem{ghassami2018fairness}
AmirEmad Ghassami, Sajad Khodadadian, and Negar Kiyavash.
\newblock Fairness in supervised learning: An information theoretic approach.
\newblock In {\em 2018 IEEE International Symposium on Information Theory
  (ISIT)}, pages 176--180. IEEE, 2018.

\bibitem{goodman2017european}
Bryce Goodman and Seth Flaxman.
\newblock European union regulations on algorithmic decision-making and a
  “right to explanation”.
\newblock {\em AI Magazine}, 38(3):50--57, 2017.

\bibitem{grgic2016case}
Nina Grgic-Hlaca, Muhammad~Bilal Zafar, Krishna~P Gummadi, and Adrian Weller.
\newblock The case for process fairness in learning: Feature selection for fair
  decision making.
\newblock In {\em NIPS Symposium on Machine Learning and the Law}, volume~1,
  page~2, 2016.

\bibitem{hajian2013methodology}
Sara Hajian and Josep Domingo-Ferrer.
\newblock A methodology for direct and indirect discrimination prevention in
  data mining.
\newblock {\em IEEE transactions on knowledge and data engineering}, 2013.

\bibitem{hardt2016equality}
Moritz Hardt, Eric Price, Nati Srebro, et~al.
\newblock Equality of opportunity in supervised learning.
\newblock In {\em Advances in Neural Information Processing Systems}, pages
  3315--3323, 2016.

\bibitem{hunt2005redlining}
D~Bradford Hunt.
\newblock Redlining.
\newblock {\em Encyclopedia of Chicago}, 2005.

\bibitem{jiang2020identifying}
Heinrich Jiang and Ofir Nachum.
\newblock Identifying and correcting label bias in machine learning.
\newblock In {\em International Conference on Artificial Intelligence and
  Statistics}, pages 702--712. PMLR, 2020.

\bibitem{joseph2016rawlsian}
Matthew Joseph, Michael Kearns, Jamie Morgenstern, Seth Neel, and Aaron Roth.
\newblock Rawlsian fairness for machine learning.
\newblock {\em arXiv preprint arXiv:1610.09559}, 1(2), 2016.

\bibitem{joseph2016fairness}
Matthew Joseph, Michael Kearns, Jamie~H Morgenstern, and Aaron Roth.
\newblock Fairness in learning: {C}lassic and contextual bandits.
\newblock In {\em Advances in Neural Information Processing Systems}, pages
  325--333, 2016.

\bibitem{kamiran2012data}
Faisal Kamiran and Toon Calders.
\newblock Data preprocessing techniques for classification without
  discrimination.
\newblock {\em Knowledge and Information Systems}, 33(1):1--33, 2012.

\bibitem{kamishima2011fairness}
Toshihiro Kamishima, Shotaro Akaho, and Jun Sakuma.
\newblock Fairness-aware learning through regularization approach.
\newblock In {\em 2011 IEEE 11th International Conference on Data Mining
  Workshops}, pages 643--650. IEEE, 2011.

\bibitem{kazemi2018scalable}
Ehsan Kazemi, Morteza Zadimoghaddam, and Amin Karbasi.
\newblock Scalable deletion-robust submodular maximization: {D}ata
  summarization with privacy and fairness constraints.
\newblock In {\em International Conference on Machine Learning}, pages
  2549--2558, 2018.

\bibitem{kilbertus2017avoiding}
Niki Kilbertus, Mateo~Rojas Carulla, Giambattista Parascandolo, Moritz Hardt,
  Dominik Janzing, and Bernhard Sch{\"o}lkopf.
\newblock Avoiding discrimination through causal reasoning.
\newblock In {\em Advances in Neural Information Processing Systems}, pages
  656--666, 2017.

\bibitem{kinyanjui2019estimating}
Newton~M Kinyanjui, Timothy Odonga, Celia Cintas, Noel~CF Codella, Rameswar
  Panda, Prasanna Sattigeri, and Kush~R Varshney.
\newblock Estimating skin tone and effects on classification performance in
  dermatology datasets.
\newblock {\em arXiv preprint arXiv:1910.13268}, 2019.

\bibitem{kleinberg2016inherent}
Jon Kleinberg, Sendhil Mullainathan, and Manish Raghavan.
\newblock Inherent trade-offs in the fair determination of risk scores.
\newblock {\em arXiv preprint arXiv:1609.05807}, 2016.

\bibitem{kusner2017counterfactual}
Matt~J Kusner, Joshua Loftus, Chris Russell, and Ricardo Silva.
\newblock Counterfactual fairness.
\newblock In {\em Advances in Neural Information Processing Systems}, pages
  4066--4076, 2017.

\bibitem{lundberg2017unified}
Scott Lundberg and Su-In Lee.
\newblock A unified approach to interpreting model predictions.
\newblock {\em arXiv preprint arXiv:1705.07874}, 2017.

\bibitem{mahoney2007method}
John~F Mahoney and James~M Mohen.
\newblock Method and system for loan origination and underwriting, October~23
  2007.
\newblock US Patent 7,287,008.

\bibitem{mase2021cohort}
Masayoshi Mase, Art~B Owen, and Benjamin~B Seiler.
\newblock Cohort shapley value for algorithmic fairness.
\newblock {\em arXiv preprint arXiv:2105.07168}, 2021.

\bibitem{pearl2009causality}
Judea Pearl.
\newblock {\em Causality}.
\newblock Cambridge university press, 2009.

\bibitem{pedreschi2009measuring}
Dino Pedreschi, Salvatore Ruggieri, and Franco Turini.
\newblock Measuring discrimination in socially-sensitive decision records.
\newblock In {\em Proceedings of the SIAM International Conference on Data
  Mining}. SIAM, 2009.

\bibitem{shapley1953value}
Lloyd~S Shapley.
\newblock A value for n-person games.
\newblock {\em Contributions to the Theory of Games}, 2(28):307--317, 1953.

\bibitem{sun2012using}
Xin Sun, Yanheng Liu, Jin Li, Jianqi Zhu, Xuejie Liu, and Huiling Chen.
\newblock Using cooperative game theory to optimize the feature selection
  problem.
\newblock {\em Neurocomputing}, 97:86--93, 2012.

\bibitem{wang2019repairing}
Hao Wang, Berk Ustun, and Flavio Calmon.
\newblock Repairing without retraining: {A}voiding disparate impact with
  counterfactual distributions.
\newblock In {\em International Conference on Machine Learning}, pages
  6618--6627. PMLR, 2019.

\bibitem{wang2018direction}
Hao Wang, Berk Ustun, and Flavio~P Calmon.
\newblock On the direction of discrimination: An information-theoretic analysis
  of disparate impact in machine learning.
\newblock {\em arXiv preprint arXiv:1801.05398}, 2018.

\bibitem{young1985monotonic}
H~Peyton Young.
\newblock Monotonic solutions of cooperative games.
\newblock {\em International Journal of Game Theory}, 14(2):65--72, 1985.

\bibitem{zafar2017fairness}
Muhammad~Bilal Zafar, Isabel Valera, Manuel Gomez~Rodriguez, and Krishna~P
  Gummadi.
\newblock Fairness beyond disparate treatment \& disparate impact: Learning
  classification without disparate mistreatment.
\newblock In {\em Proceedings of the 26th International Conference on World
  Wide Web}, pages 1171--1180. International World Wide Web Conferences
  Steering Committee, 2017.

\bibitem{zafar2015fairness}
Muhammad~Bilal Zafar, Isabel Valera, Manuel~Gomez Rodriguez, and Krishna~P
  Gummadi.
\newblock Fairness constraints: Mechanisms for fair classification.
\newblock {\em arXiv preprint arXiv:1507.05259}, 2015.

\bibitem{zemel2013learning}
Rich Zemel, Yu~Wu, Kevin Swersky, Toni Pitassi, and Cynthia Dwork.
\newblock Learning fair representations.
\newblock In {\em Proceedings of the 30th International Conference on Machine
  Learning (ICML-13)}, pages 325--333, 2013.

\end{thebibliography}

\clearpage

\appendix
\begin{center}
    {\LARGE \bf Supplementary Material}
\end{center}

\section{Proof of Theorem \ref{thm:1}}\label{appendix_thm1}
\textbf{If part:} From \eqref{eq:GM}, $Y$ is a descendant of all the features in the graphical model. Assuming $X_S\subseteq X^n$ is not a direct parent of $Y$ in the graphical model, given $X_{S^c}$, $Y$ and $X_S$ are independent, i.e., $P_{Y|A,X_{S^c},X_{S}} = P_{Y|A,X_{S^c}}$. We have
\begin{align*}
&I(Y;A,X^n) = I(Y;A,X_{S^c}, X_{S}) \\
&= \sum_{y,a,x^n} P_{Y,A,X_{S^c},X_S}(y,a,x_{S^c},x_S)\; \log \frac{P_{Y|A, X_{S^c}, X_S}(y|a,x_{S^c},x_S)}{P_Y(y)}   \\
&= \sum_{y,a,x^n} P_{Y,A, X_{S^c},X_S}(y,a,x_{S^c},x_{S})\;\log \frac{P_{Y|A,X_{S^c}}(y|a,x_{S^c})}{P_Y(y)} \\
&= \sum_{y,a,x_{S^c}} P_{Y,A,X_{S^c}}(y,a,x_{S^c})\log \frac{P_{Y|A,X_{S^c}}(y|a,x_{S^c})}{P_Y(y)}\\ 
&=I(Y;A,X_{S^c})\\
&\qquad \implies v^{Acc}(X_S) = I(Y; A, X_S| X_{S^c}) = 0.
\end{align*}\\ \\
\textbf{Only if part:} Now assume $v^{Acc}(X_S) = 0$. We want to prove that $X_S$ is not a direct parent of $Y$.
\begin{align*}
&v^{Acc}(X_S) = 0 \\
&\implies I(Y;X^n,A)=I(Y; X_{S^c},A)\\
&\implies  \sum_{y,x^n,a} P_{Y,X^n,A}(y,x^n,a)\log \frac{P_{Y|X^n,A}(y|x^n,a)}{P_Y(y)} \\
&\quad = \sum_{y,x_{S^c},a} P_{Y,X_{S^c}, A}(y,x_{S^c},a)\log \frac{P_{Y|X_S^c,A}(y|x_{S^c},a)}{P_Y(y)}\\
&\quad = \sum_{y,x^n,a} P_{Y,X^n,A}(y,x^n,a)\log \frac{P_{Y|X_{S^c},A}(y|x_{S^c},a)}{P_Y(y)}\\
&\implies \sum_{y,x^n,a} P_{Y,X^n,A}(y,x^n,a) \log \frac{P_{Y|X^n,A}(y|x^n,a)}{P_{Y|X_{S^c},A}(y|x_{S^c},a)}=0\\
&\implies \hspace{-3mm} \sum_{a,x^n\in\mathcal{X}^n} \hspace{-2mm} P_{A,X^n}(a,x^n)   D_{KL}(P_{Y|A=a,X^n=x^n}||P_{Y|A=a,X_{S^c} = x_{S^c}})= 0
\end{align*}
Thus, $D_{KL}(P_{Y|A=a,X^n=x^n}||P_{Y|A=a,X_{S^c} = x_{S^c}}) = 0$, for all $a,x^n$, since $D_{KL}(.||.) \geq 0$. Furthermore, according to the Gibbs' inequality, $D_{KL}(P||Q)=0$ if and only if $P=Q$ almost every where. As a result, $P_{Y|A=a,X^n=x^n}=P_{Y|A=a,X_{S^c} = x_{S^c}}$, for all $a$, and $x^n\in\mathcal{X}^n \implies Y\perp X_S| \{A,X_{S^c}\}$. Since we assume that the joint distribution is faithful with respect to the graphical model, we conclude that $X_S$ is not a direct parent of $Y$.

\section{Proof of Proposition \ref{prop:1}}\label{appendix_prop1}
According to \cite{bertschinger2014quantifying}, unique information is monotonic with respect to its third argument, i.e.,
\begin{align*}
S_1&\subseteq S_2 \implies X_{S_1}\subseteq X_{S_2} \\
&\implies UI(Y;A\setminus X_{S_1}) \geq UI(Y;A\setminus X_{S_2}).
\end{align*}
In addition, we have:
\[
SI(Y;A,X_S) = I(Y;A) - UI(Y;A\setminus X_S)
\]
As a result:
\[
S_1\subseteq S_2 \implies SI(Y;A,X_{S_1}) \leq SI(Y;A,X_{S_2}).
\]
The proof of non-negativity of shared information can be found in \cite{bertschinger2014quantifying}. Further, both mutual information $I(A;X_S)$, and conditional mutual information $I(A;X_S|Y)$ are non-negative and monotonic with respect to $X_S$ \cite{cover2012elements}. Since all three terms $SI(Y;X_S,A)$, $I(A;X_S)$, and $I(A;X_S|Y)$ are non-negative and monotonic in $X_S$, their multiplication is also non-negative and monotonic in $X_S$.

\section{Proof of Proposition \ref{prop:2}}\label{appendix_prop2}
$X_S\independent A$ and $X_S\independent A|Y$ imply that $I(X_S;A)=0$ and $I(X_S;A|Y) = 0$, respectively, see \cite{cover2012elements}. Furthermore, as shown in \cite[Lemma 12]{bertschinger2014quantifying}: 
\[
SI(Y;X_S,A) \leq I(Y;X_S).
\]
As a result,
\begin{align*}
Y\independent X_S \implies & SI(Y;X_S,A) = 0 \\
\implies & v^D(X_S) = 0.    
\end{align*}

\section{Proof of Theorem \ref{thm: 2}}\label{appendix_thm2}
According to \cite[Lemma~13]{bertschinger2014quantifying}, for a path-discriminatory subset of features $X_S$, we have $SI(Y;X_S,A) = I(Y;A)$. 

According to \cite[Lemma~12]{bertschinger2014quantifying}, for any $X_S\subseteq X^n$, $SI(Y;X_S,A) \leq I(Y;A)$. Thus, the discrimination coefficient of a path-discriminatory subset of features is an upper bound of the discrimination coefficient of all other subsets of features, and hence Theorem \ref{thm: 2} follows.

\section{Proof of Corollary \ref{cor:1}}\label{appendix_cor1}
The accuracy and discrimination coefficients $v^{Acc}(.)$ and $v^{D}(.)$ both are non-negative and monotonic, see Proposition \ref{prop:1}.
The marginal accuracy and discrimination coefficients $\phi_i^{Acc}$ and $\phi_i^D$, $i\in[n]$, are the summation of terms on the form $v^{Acc/D}(X_S \cup \{X_i\}) - v^{Acc/D}(X_S)$, which are non-negative due to the monotonicity and non-negativity of $v^{Acc}$ and $v^{D}$. 

\section{Proof of Corollary \ref{cor:acc}} \label{appendix_cor2}
For the marginal accuracy coefficient $\phi_i^{Acc}$, $i\in[n]$, we have
\begin{align*}
&\phi_i^{Acc}=\sum_{X_S\subseteq X^n_{- i}} w(X_S) \left[v^{Acc}(X_S \cup \{X_i\}) - v^{Acc} (X_S)\right],
\end{align*}
where $w(X_S)=\frac{|X_S|!(n-|X_S|-1)!}{n!}$ and $X^n_{- i} = X^n \setminus \{X_i\}$. Recall, $v^{Acc}(\cdot)$ is defined in \eqref{eq:vAcc}. We thus have 
\begin{align*}
\phi_i^{Acc}&= \sum_{X_S\subseteq X^n_{-i}}  w(X_S) \Big[ I\left(Y; X_S\cup\{X_i\} \big| A,(X_{S^c} \setminus X_i )\right)- (I(Y;X_{S}|A,X_{S^c})\Big]\\
&=  \sum_{X_S\subseteq X^n_{-i}}  w(X_S) \Big[I(Y; A, X^n) - I\left(Y; A, X_{S^c}\setminus \{X_i\}\right)- I(Y; A, X^n) + I(Y; A, X_{S^c})\Big]\\
&= \sum_{X_S\subseteq X^n_{-i}}  w(X_S) \Big[I(Y; A, X_{S^c}) - I\left(Y; A, X_{S^c}\setminus \{X_i\}\right)\Big]\\
&= \sum_{X_S\subseteq X^n_{-i}}  w(X_S) \Big[I(Y; A, X_S \cup \{X_{i}\}) - I\left(Y; A, X_S\right)\Big].
\end{align*}
The last equality follows by applying a change of variables, where we replace $(X_{S^c}\setminus \{X_i\})$ by $X_S$, and thus $X_{S^c}$ is replaced by $X_S \cup \{X_i\}$.

Suppose $i\in [n]\setminus a$.
For $X_S\subseteq X^n\setminus \{X_i,X_a\}$, we have: 
\begin{align*}
I(Y;A,X_S\cup\{X_a\}) &= I(Y;X_a) + I(Y;A,X_S|X_a) \\
&= I(Y;X_a),    
\end{align*}
where the last equality follows from the fact that $X_a$ is the only parent of $Y$, and hence $I(Y;A,X_S|X_a)=0$. Therefore, for any $X_S$ such that $X_a\in X_S$, we have
\begin{align*}
&I(Y;A,X_S\cup\{X_i\}) =  I(Y;A,X_S) = I(Y;X_a) \\
&\implies I(Y;A,X_S\cup\{X_i\}) -  I(Y;A,X_S) = 0.
\end{align*}
As a result, $\phi_a^{Acc} \geq \phi_i^{Acc}, \forall i\in [n]\setminus a$.

\section{Proof of Corollary \ref{cor:allayp}}\label{appendix_cor3}
Suppose that $X_d$ satisfies the condition of Corollary \ref{cor:allayp}:  $X_d$ is a single child of $A$, and $A\independent (Y\cup \{X^n \setminus X_d\})|X_d$. Also, suppose that $i\in [n]\setminus d$. 

For $X_S\subseteq X^n \setminus \{X_i\}$, if $X_d\in X_S$, we have $v^D(X_S) = v^D(X_S \cup \{X_i\}) = I(Y;A) I(A;X_d) I(A;X_d|Y)\implies v^D(S\cup \{X_i\}) - v^D(X_S) = 0$. In addition, if $X_d\notin X_S$, due to the Theorem \ref{thm: 2} and data processing inequality \cite{cover2012elements}, we have $v^D(X_S\cup \{X_i\})- v^D(X_S) \leq v^D(X_S\cup \{X_d\})- v^D(X_S)$. 

As in Appendix \ref{appendix_cor2}, let $w(X_S)=\frac{|X_S|!(n-|X_S|-1)!}{n!}$. For all $i \in [n]\setminus d$, we have
\begin{align*}
\phi^D_i &=\sum_{X_S\subseteq X^n\setminus\{X_i\}} w(X_S)\;(v^D(X_S\cup \{X_i\}) - v^D(X_S)) \\
&\leq \sum_{X_S\subseteq X^n \setminus\{X_d\}} w (X_S)\;(v^D(X_S\cup \{X_k\}) - v^D(X_S)) \\
&= \phi^D_d.
\end{align*}

\end{document}